\documentclass{article}

\usepackage{microtype}
\usepackage{graphicx}
\usepackage{booktabs, multirow, array} % tables
\usepackage{url}
\usepackage{xcolor}
\usepackage{listings}
\usepackage{hyperref}

\usepackage[preprint]{icml2026}
\makeatletter
\renewcommand{\Notice@String}{}
\makeatother

\usepackage[utf8]{inputenc}
\usepackage[T1]{fontenc}
\usepackage{amsmath, amssymb, amsthm, mathtools}
\hypersetup{
  colorlinks=true,
  linkcolor=blue,
  citecolor=blue,
  urlcolor=blue
}

\theoremstyle{plain}
\newtheorem{theorem}{Theorem}

\newtheorem{proposition}{Proposition}
\newtheorem{corollary}{Corollary}

\theoremstyle{definition}
\newtheorem{definition}{Definition}

\theoremstyle{remark}

\newcommand{\E}{\mathbb{E}}
\newcommand{\Pp}{\mathbb{P}}

\newcommand{\Ind}{\mathbf{1}}
\newcommand{\KL}{\mathrm{KL}}
\newcommand{\MI}{\mathrm{I}}
\newcommand{\Hh}{\mathrm{H}}
\newcommand{\ber}{\mathrm{Ber}}
\newcommand{\doop}{\mathrm{do}}

\newcommand{\Cset}{\mathcal{C}}
\newcommand{\Vocab}{\mathcal{V}}

\newcommand{\GateGap}{\mathrm{GateGap}}
\newcommand{\ValueGap}{\mathrm{ValueGap}}

\providecommand{\cref}[1]{\ref{#1}}
\providecommand{\Cref}[1]{\ref{#1}}

\icmltitlerunning{Attention Deficits in Language Models}

\begin{document}
	\twocolumn[
		  \icmltitle{Attention Deficits in Language Models:\\
		  Causal Explanations for Procedural Hallucinations}
			  \begin{icmlauthorlist}
			    \icmlauthor{Ahmed Karim}{hl}
			    \icmlauthor{Fatima Sheaib}{hl}
			    \icmlauthor{Zein Khamis}{hl}
			    \icmlauthor{Maggie Chlon}{hl}
			    \icmlauthor{Jad Awada}{hl}
			    \icmlauthor{Leon Chlon}{hl,ox}
			  \end{icmlauthorlist}
			  \icmlaffiliation{hl}{Hassana Labs}
			  \icmlaffiliation{ox}{University of Oxford}
			  \icmlcorrespondingauthor{Leon Chlon}{leo@hassana.io}
			  \icmlkeywords{Long context, hallucination, interpretability}
			  \vskip 0.3in
			]
\printAffiliationsAndNotice{}  % required even if empty

\begin{abstract}
Large language models can follow complex procedures yet fail at a seemingly trivial final step: reporting a value they themselves computed moments earlier. We study this as \emph{procedural hallucination}: failure to execute a \emph{verifiable, prompt-grounded specification} even when the correct value is present in-context. In long-context binding tasks with a known single-token candidate set, we find that many errors are \emph{readout-stage routing failures}. Specifically, failures decompose into \emph{Stage~2A (gating)} errors where the model does not enter answer mode, and \emph{Stage~2B (binding)} errors where it enters answer mode but selects the wrong candidate (often due to recency bias). In the hard regime, Stage~2B accounts for most errors across model families in our tasks (Table~\ref{tab:stage2ab_multimodel}). On Stage~2B error trials, a linear probe on the final-layer residual stream recovers the correct value far above chance (e.g., 74\% vs.\ 2\% on Qwen2.5-3B; Table~\ref{tab:probe_stage2b}), indicating the answer is encoded but not used. We formalize ``present but not used'' via available vs.\ used mutual information and pseudo-prior interventions, yielding output-computable diagnostics and information-budget certificates. Finally, an oracle checkpointing intervention that restates the true binding near the query can nearly eliminate Stage~2B failures at long distance (e.g., Qwen2.5-3B 0/400$\to$399/400 at $k=1024$; Table~\ref{tab:checkpoint_results_full}).
\end{abstract}

\section{Introduction}

When a language model fabricates a historical date or invents a citation, we call this a \emph{factual hallucination}: the model lacks the relevant knowledge. But a different failure mode is equally important and far less understood. Consider the \textbf{strawberry counting task} (see Appendix~\ref{sec:strawberry_appendix}): \textbf{``How many r's are in strawberry?''} \textbf{The model shows its work via chain-of-thought}, enumerating letters and maintaining a running count. For short traces, this works. But as the trace grows longer, the final answer drifts: at length $k=20$, the model outputs 6; at $k=30$, it outputs 8. The computation was faithful throughout. The failure occurred at the final step, where the readout mechanism selected a wrong statistic (a cumulative or recency-weighted count) instead of the intended one. The model did not lack the information; it simply failed to use it.

We call this a \textbf{procedural hallucination}: a failure to follow a \emph{verifiable prompt-grounded specification}. Formally, we study settings where a specification function $g(W)$ deterministically maps a prompt $W$ to a single-token answer in a known candidate set, and a procedural hallucination occurs when the model outputs $\hat Y \neq g(W)$ despite $g$ being unambiguous (Definition~\ref{def:procedural_hallucination}). Unlike factual hallucination, which concerns world knowledge, procedural hallucination concerns the model's ability to faithfully execute a specification given in its own context. When the required evidence is already present in-context, scaling or retrieval is not sufficient on its own: the remaining challenge is \emph{routing} that evidence to the output.

This is not a knowledge failure. The model computed the correct value; probing its hidden state on error trials recovers the answer with high accuracy. The information is present; it is not routed to the output. This observation raises two questions. First, can we decompose these failures into interpretable stages? Second, can we formalize ``information is present but not used'' in a way that yields provable guarantees? This paper answers both.

\subsection{Overview of Contributions}

We develop a theory of procedural hallucinations grounded in information theory and causal reasoning, then validate it empirically across model families.

\paragraph{Results at a glance (core claims).}
\begin{itemize}
\item \textbf{Stage~2B dominates in the hard regime:} when these binding tasks fail, the model usually \emph{enters answer mode} but selects the \emph{wrong candidate} (Stage~2B), rather than refusing to answer (Stage~2A) (Table~\ref{tab:stage2ab_multimodel}, Figure~\ref{fig:spotlight_results}).
\item \textbf{Information is present on error trials:} on Stage~2B errors, a linear probe on the final-layer residual stream recovers the correct value far above chance (e.g., 74\% vs.\ $\approx$2\% on Qwen2.5-3B) (Table~\ref{tab:probe_stage2b}).
\item \textbf{Checkpointing restores long-context binding:} restating the true binding near the query can convert near-zero baseline accuracy into near-perfect accuracy at long distance (e.g., Qwen2.5-3B 0/400$\to$399/400 at $k=1024$) (Table~\ref{tab:checkpoint_results_full}, Figure~\ref{fig:spotlight_results}).
\end{itemize}

\paragraph{Scope.} Our main empirical results use synthetic long-context binding tasks (\textsc{competing\_vars}, \textsc{primacy\_recency}) where the target is a single token from a known candidate set (Definition~\ref{def:procedural_hallucination}); we additionally include one option-randomized naturalistic variant (\texttt{notes\_binding}; Appendix Table~\ref{tab:notes_binding_scale}). We do not claim that the same mechanisms fully explain all long-context failures in naturalistic tasks such as document QA; we treat that as an open validation target (Limitations, Section~\ref{sec:limitations}).

\paragraph{Contributions (with evidence pointers).}
\begin{itemize}
\item \textbf{Stagewise decomposition + diagnostics.} We define logprob-derived margins ($\GateGap$, $\ValueGap$) that classify each error as Stage~2A (gating) vs.\ Stage~2B (binding), and show Stage~2B dominates errors in the hard regime across model families (Definition~\ref{def:procedural_hallucination}, Table~\ref{tab:stage2ab_multimodel}).
\item \textbf{Routing-efficiency theory.} We formalize ``present but not used'' via available vs.\ used information ($I_{\mathrm{avail}}$, $I_{\mathrm{used}}$) and derive certificates and slack decompositions connecting error rates to information budgets (Section~\ref{sec:theory}, Appendix proofs).
\item \textbf{Pseudo-priors as causal baselines.} We introduce pseudo-priors induced by evidence-scrubbing interventions and prove tight ``bits-to-trust'' lower bounds for overcoming model biases (Section~\ref{sec:pseudo_priors}, Theorem~\ref{thm:bernoulli_projection}).
\item \textbf{Empirical + mechanistic validation and mitigation.} We show (i) probes recover the correct value on Stage~2B errors (Table~\ref{tab:probe_stage2b}), (ii) activation patching localizes failures to late components (Section~\ref{sec:empirical}), and (iii) checkpointing recovers accuracy at long distance (Table~\ref{tab:checkpoint_results_full}).
\end{itemize}
\paragraph{Artifact.} An anonymized reproducibility package (scripts + cached outputs) is included in the supplemental material; see Appendix Section~\ref{sec:toolkit}.

\begin{figure}[t]
\centering
\includegraphics[width=0.90\columnwidth]{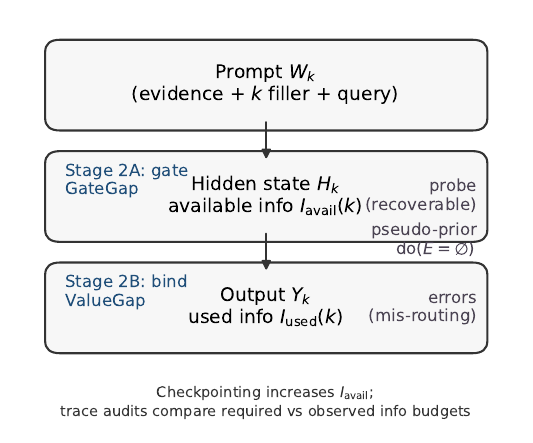}
\caption{Framework overview: Stage~2A gating and Stage~2B binding failures correspond to low routing efficiency $I_{\mathrm{used}}/I_{\mathrm{avail}}$, diagnosed via pseudo-prior interventions.}
\label{fig:procedural_overview}
\end{figure}

\section{Related Work}

Procedural hallucinations relate to long-context failures where the answer is explicitly present yet not produced (``know but don't tell'') \cite{liu2024lost,lu2024know,wu2024retrieval} and to broader hallucination/faithfulness taxonomies \cite{ji2023survey,maynez2020faithfulness}. Our diagnostics draw on mechanistic interpretability (causal tracing/editing, circuit discovery, and binding analyses) \cite{patching_general,conmy2023automated,wang2023interpretability,elhage2021mathematical} and information-theoretic frameworks for representation and bottlenecks \cite{tishby1999information,belghazi2018mine,coverthomas}. We operationalize these threads into an API-compatible stagewise decomposition, routing efficiency metric, and pseudo-prior information budgets.

\section{Stagewise Slot Population}
\label{sec:stage}

We now formalize the setting in a way that is both precise and operationally compatible with API-accessible outputs. The key insight is that slot-population errors can fail in two distinct ways, and distinguishing them clarifies both diagnosis and mitigation.
We refer to these as Stage~2A/2B to emphasize that they are \emph{readout-stage} failures (Stage~2): they occur after an implicit Stage~1 in which the model computes/encodes the correct value in its hidden state.

\begin{definition}[Procedural hallucination]
\label{def:procedural_hallucination}
Fix a candidate set $\Cset\subset \Vocab$ and a specification function $g$ that deterministically maps a prompt $W$ to a target value $V=g(W)\in\Cset$.
Let $\hat Y(W)$ denote the model's next-token output at the readout position under greedy decoding.
We say the model exhibits a \emph{procedural hallucination} on $W$ if $\hat Y(W)\neq g(W)$ in settings where $g(W)$ is unambiguous and verifiable from the prompt (i.e., the answer is deterministically recoverable by an explicit procedure).
\end{definition}

\subsection{Prompt Structure}

A prompt $W_k$ has three components: (i) a binding region where keys are assigned values, (ii) $k$ filler tokens, and (iii) a query requesting a specific value. Values come from a candidate set $\Cset\subset \Vocab$ of single-token strings. We study two task variants:

\begin{table}[h]
\centering
\scriptsize
{\setlength{\tabcolsep}{3pt}
\begin{tabular}{@{}p{0.47\linewidth}p{0.47\linewidth}@{}}
\toprule
\textbf{\textsc{competing\_vars}} & \textbf{\textsc{primacy\_recency}} \\
\midrule
\texttt{KEY1 = [apple]} & \texttt{KEY = [alpha]} \\
\textit{[k filler tokens]} & \textit{[k filler tokens]} \\
\texttt{KEY2 = [banana]} & \texttt{KEY = [beta]} \\
\textit{[k filler tokens]} & \textit{[k filler tokens]} \\
\texttt{What is KEY1? KEY1 = [} & \texttt{KEY = [gamma]} \\
 & \textit{[k filler tokens]} \\
 & \texttt{What was the FIRST value of KEY? KEY = [} \\
\midrule
\multicolumn{1}{p{0.47\linewidth}}{\small Different keys hold different values; the competitor is a distractor.} &
\multicolumn{1}{p{0.47\linewidth}}{\small Same key reassigned; recency bias actively competes with the correct answer.} \\
\bottomrule
\end{tabular}
}
\end{table}

\subsection{Stage~2A: Does the Model Enter Answer Mode?}

Let $z_\theta(x\mid W_k)$ denote the next-token logit for token $x$ at the readout position. We define:
\begin{align}
Y_k &:= \arg\max_{v\in \Cset} z_\theta(v\mid W_k) && \text{(best candidate)} \\
\hat Y_k &:= \arg\max_{x\in \Vocab} z_\theta(x\mid W_k) && \text{(best overall token)}
\end{align}
The indicator $G_k:=\Ind\{\hat Y_k\in \Cset\}$ marks whether the model ``enters answer mode'' by outputting a candidate token. We quantify this with the \textbf{gate margin}:
\[
\GateGap(W_k)\;:=\;\max_{v\in \Cset} z_\theta(v\mid W_k)\;-\;\max_{x\notin \Cset} z_\theta(x\mid W_k).
\]
When $\GateGap > 0$, the model will output a candidate; when $\GateGap < 0$, it will output something else entirely (e.g., punctuation or a hedging phrase). Negative gate margin indicates a Stage~2A failure.

\subsection{Stage~2B: Does the Model Select the Right Candidate?}

Conditional on entering answer mode, we ask whether the model selects correctly. Let $V\in \Cset$ denote the ground-truth value. The \textbf{value margin} is:
\[
\ValueGap(W_k)\;:=\; z_\theta(V\mid W_k)\;-\;\max_{v\in \Cset\setminus\{V\}} z_\theta(v\mid W_k).
\]
Positive value margin means correct binding; negative value margin means the model prefers a wrong candidate. This is a Stage~2B failure.

\subsection{Why This Decomposition Matters}

The distinction between Stage~2A and Stage~2B is not merely taxonomic. The two failure modes have different signatures and different remedies:
\begin{itemize}
\item \textbf{Stage~2A failures} indicate that the model does not recognize the query as requesting a value from $\Cset$. This is a format or instruction-following problem.
\item \textbf{Stage~2B failures} indicate that the model understands the task format but routes to the wrong value. This is an information-routing problem, often driven by recency bias.
\end{itemize}
We report global accuracy $\Pp(\hat Y_k = V)$, candidate accuracy $\Pp(Y_k = V)$, and an \emph{error decomposition} into gating vs.\ binding failures via
$\mathrm{Frac}\text{-}2A := \Pp(G_k=0 \mid \hat Y_k\neq V)$ and $\mathrm{Frac}\text{-}2B := \Pp(G_k=1 \mid \hat Y_k\neq V)$ (so $\mathrm{Frac}\text{-}2A+\mathrm{Frac}\text{-}2B=1$).

\section{Information-Theoretic Framework}
\label{sec:theory}

We now formalize ``information is present but not used.'' The key objects are available information (what the hidden state knows) and used information (what the output exploits). Proofs appear in \cref{sec:proofs}.

\subsection{Available versus Used Information}

Let $H_k$ denote the model's internal state at the readout position (e.g., the final-layer residual stream). The data-generating process forms a Markov chain:
\[
V \;\to\; W_k \;\to\; H_k \;\to\; Y_k,
\]
where $V$ is the ground-truth value, $W_k$ is the prompt, $H_k$ is the hidden state, and $Y_k$ is the model's candidate-set decision.

\begin{definition}[Available and used information]
\label{def:avail_used}
We define:
\[
I_{\mathrm{avail}}(k) := \MI(V;H_k),\qquad
I_{\mathrm{used}}(k) := \MI(V;Y_k),
\]
and the \textbf{routing efficiency}:
\[
\eta_k := \frac{I_{\mathrm{used}}(k)}{I_{\mathrm{avail}}(k)}\in[0,1].
\]
\end{definition}

\begin{proposition}[Data processing]
\label{prop:dpi_eta}
For any $k$, we have $0\le \eta_k \le 1$.
\end{proposition}

Routing efficiency captures how much of the available information the model actually exploits. Procedural hallucinations correspond to $\eta_k \ll 1$: the hidden state encodes the answer ($I_{\mathrm{avail}}$ is large), but the output ignores it ($I_{\mathrm{used}}$ is small).
In practice, $I_{\mathrm{used}}$ can be lower-bounded from candidate-set error rates via Fano, and $I_{\mathrm{avail}}$ can be lower-bounded (when activations are available) via probe performance, yielding an empirical routing-certificate summary (Appendix Table~\ref{tab:routing_mi_bounds}).

\subsection{From Error Rates to Information: Fano Bounds}

Let $M:=|\Cset|$ and define the candidate-set error rate $\varepsilon(k) := \Pp(Y_k\neq V)$.

\begin{theorem}[Fano lower bound]
\label{thm:fano}
If $V$ is uniform on $\Cset$, then:
\[
I_{\mathrm{used}}(k)
\;\ge\;
\log M\;-\;h(\varepsilon(k))\;-\;\varepsilon(k)\log(M-1),
\]
where $h(\cdot)$ is binary entropy in nats.
\end{theorem}

This bound says that high error implies low used information. But is the bound tight? The following proposition says yes, and characterizes exactly when.

\begin{proposition}[Minimax tightness]
\label{prop:mano_tight}
For every $M$ and $\varepsilon\in[0,1-1/M]$, the $M$-ary symmetric channel achieves equality in \cref{thm:fano}. This channel is minimax optimal: it minimizes $\MI(V;Y)$ among all channels with error rate $\varepsilon$.
\end{proposition}

When does the bound have slack? The following decomposition makes this precise.

\begin{proposition}[Fano slack decomposition]
\label{prop:fano_slack}
Under the conditions of \cref{thm:fano}, define $E=\Ind\{Y\neq V\}$. Then:
\begin{equation}\label{eq:fano_slack_decomp}
\begin{aligned}
\MI(V;Y)
&=
\underbrace{\log M - h(\varepsilon) - \varepsilon\log(M-1)}_{\text{Fano lower bound}} \\
&\quad+\underbrace{\big(h(\varepsilon)-\Hh(E\mid Y)\big)}_{\text{Jensen slack }\ge 0} \\
&\quad+\underbrace{\varepsilon\cdot\big(\log(M-1)-\Hh(V\mid Y,E=1)\big)}_{\text{non-uniform confusion slack }\ge 0}.
\end{aligned}
\end{equation}
The first slack term vanishes iff the error rate is constant across output values. The second vanishes iff, given an error, all wrong values are equally likely.
\end{proposition}

\begin{corollary}[Fano inversion]
\label{cor:fano_invert}
Define $\Phi_M(\varepsilon):=\log M-h(\varepsilon)-\varepsilon\log(M-1)$. Then $\MI(V;Y)\ge \Phi_M(\varepsilon)$ implies $\varepsilon \ge \Phi_M^{-1}(\MI(V;Y))$. Under the $M$-ary symmetric channel, this is an equality.
\end{corollary}

\subsection{Pseudo-Priors and Decompression Bounds}
\label{sec:pseudo_priors}

The Fano bounds connect error to used information. But how much information does the model \emph{need}? This depends on its baseline bias. A model with strong recency bias needs more evidence to overcome that bias than an unbiased model would.

We formalize this via a causal intervention that removes the binding evidence.

\begin{definition}[Pseudo-prior]
\label{def:pseudoprior}
Let $E$ denote the binding evidence (e.g., ``KEY1 = [apple]''). Define a null distribution by the intervention:
\[
\tilde W_k \sim \doop(E=\varnothing),
\]
which removes $E$ while preserving the template, candidate set, and competing cues (e.g., recency). Let $\tilde Y_k$ be the model's decision under $\tilde W_k$. The \textbf{pseudo-prior} is:
\[
\tilde p_k := \Pp(\tilde Y_k = V).
\]
\end{definition}

\subsubsection{Operationalizing $\doop(E=\varnothing)$: structure-preserving evidence ablation}
\label{sec:doe_empty_operationalization}

The intervention $\doop(E=\varnothing)$ is a \emph{design pattern}: remove the \emph{semantic support} for a binding while preserving the \emph{structure} of the instance so that differences in behavior can be attributed to the missing evidence rather than to superficial prompt changes.

\paragraph{Design goal (invariances of the null).}
Our null operator is chosen to preserve:
(i) the prompt template (role labels, span IDs, delimiters, formatting),
(ii) the query and candidate set,
and (iii)---as far as practical---length and locality statistics (so that ``distance'' and recency cues remain comparable).
This answers the ``empty string vs.\ noise vs.\ alternative binding'' ambiguity: we target \emph{structure-preserving evidence ablation}, not an arbitrary prompt corruption.

\paragraph{Operator used in experiments.}
In binding experiments, we implement $\doop(E=\varnothing)$ by removing the key--value \emph{content} while keeping the key line and delimiters (e.g., replacing \texttt{KEY1 = [apple]} with \texttt{KEY1 = [REDACTED]}).
In trace auditing (\cref{sec:cot}), we implement $\doop(E=\varnothing)$ as \emph{span scrubbing}: for a step citing spans $S_i$, we replace the \emph{contents} of those spans with a fixed placeholder (default \texttt{[REDACTED]}) while preserving span labels and delimiters, and re-run the verifier on the scrubbed prompt to obtain $p_{0,i}$.

\paragraph{What the pseudo-prior means (and why \texttt{[REDACTED]} is not ``cheating'').}
We emphasize that \texttt{[REDACTED]} is not distribution-neutral. It is an explicit marker for ``evidence removed.''
Accordingly, $\tilde p$ should be interpreted as \emph{the model/verifier's probability when explicitly denied access to that evidence}, which is exactly the counterfactual required for budgeting.
Empirically, this tends to be conservative: scrubbing usually drives $p_0$ downward (more \texttt{UNSURE}/abstain behavior), which increases the required bits-to-trust and therefore \emph{flags more}, not fewer, instances.

\paragraph{Robustness via a null family (envelope certification).}
To avoid over-committing to a single null implementation, we can define a small family of structure-preserving null operators $\mathcal{N}$ (e.g., delete-span, redact-span, same-length masking), and compute a pseudo-prior interval
\[
p_{0}^{\min} := \min_{\nu\in\mathcal{N}} p_0^{(\nu)}, \qquad
p_{0}^{\max} := \max_{\nu\in\mathcal{N}} p_0^{(\nu)}.
\]
For budget tests of the form $\KL(\ber(p_1)\|\ber(p_0)) \ge \KL(\ber(\tau)\|\ber(p_0))$, we then certify against the \emph{hardest} null (worst-case over $\nu\in\mathcal{N}$), i.e.\ require the inequality to hold for all $\nu\in\mathcal{N}$. This ``envelope'' move is directly analogous to the assumption-free, output-computable robustness layer used in compression-based ISR planners: we do not claim a single null is uniquely correct, we certify against a small, explicit family.

The pseudo-prior measures how likely the model is to guess correctly without the critical evidence. If recency bias is strong and the correct answer appeared first, $\tilde p_k$ will be small.

\begin{theorem}[Bernoulli-projected decompression bound]
\label{thm:bernoulli_projection}
Let $P$ and $Q$ be distributions with $P(A)=p$ and $Q(A)=\tilde p$ for some event $A$. Then:
\[
\KL(P\|Q) \;\ge\; \KL(\ber(p)\,\|\,\ber(\tilde p)).
\]
Moreover, for any $(p,\tilde p)$ there exists a pair $(P,Q)$ achieving equality.
\end{theorem}

\begin{corollary}[Bits-to-trust]
\label{cor:bits_to_trust}
To achieve success probability $p$ from pseudo-prior $\tilde p$, the model needs at least $\KL(\ber(p)\|\ber(\tilde p))$ nats.
\end{corollary}

\paragraph{Concrete example.} Suppose recency bias gives the correct answer a pseudo-prior of $\tilde p = 0.05$. To achieve 90\% accuracy, the model needs at least $\KL(\ber(0.9)\|\ber(0.05)) \approx 2.25$ nats $\approx 3.2$ bits of evidence. This is the ``bits-to-trust'' cost of overcoming the bias.

\subsection{Certifying ``Present but Not Used''}

Estimating $I_{\mathrm{avail}} = \MI(V;H_k)$ requires access to hidden states. But we can certify routing failure without a tight estimate.

\begin{proposition}[Certification via probing]
\label{prop:unused_info}
For any probe $f$:
\[
\MI(V;H_k)-\MI(V;Y_k)\;\ge\;\MI(V;f(H_k))-\MI(V;Y_k).
\]
\end{proposition}

If a linear probe achieves higher mutual information with $V$ than the model's own output does, then the model is provably failing to use available information. We do not need to estimate $I_{\mathrm{avail}}$ exactly; we only need a probe that outperforms the model.

\subsection{Distance Dependence}
\label{sec:sdpi_main}

Why do binding failures worsen with distance? We formalize this via strong data processing inequalities (SDPI).

\begin{definition}[SDPI coefficient]
For a channel $K$ mapping distributions on $\mathcal{X}$ to distributions on $\mathcal{X}'$:
\[
\alpha(K)
:= \sup_{P\neq Q}\ \frac{\KL(PK\,\|\,QK)}{\KL(P\,\|\,Q)} \;\in\; [0,1].
\]
\end{definition}

\begin{theorem}[SDPI contraction]
\label{thm:sdpi}
For a Markov chain $U\to X\to X'$ with channel $K$:
\[
\MI(U;X') \;\le\; \alpha(K)\,\MI(U;X).
\]
For a chain $V\to S_0\to S_1\to \cdots \to S_k\to H_k$:
\[
\MI(V;H_k)\ \le\ \Big(\prod_{t=1}^{k} \alpha(K_t)\Big)\,\MI(V;S_0).
\]
\end{theorem}

Information decays geometrically with distance. This is not merely an upper bound; for ``copy-or-noise'' channels, the decay is exact (\cref{prop:copy_noise} in Appendix).

\paragraph{Implication for checkpointing.} Checkpointing injects fresh copies of the evidence, resetting the distance. This increases $I_{\mathrm{avail}}(k)$ and thus the attainable accuracy.

\section{Empirical Results}
\label{sec:empirical}

Our experiments test three predictions of the framework:
\begin{enumerate}
\item Stage~2B errors (wrong candidate) should dominate over Stage~2A errors (no candidate) in binding tasks.
\item On Stage~2B error trials, probes should recover the correct answer at above-chance rates, certifying ``present but not used.''
\item Checkpointing should recover accuracy by shortening the effective evidence distance.
\end{enumerate}
We also localize the failure mechanistically via activation patching. Full protocols appear in \cref{sec:protocols}.

\begin{figure*}[t]
\centering
\includegraphics[width=0.98\textwidth]{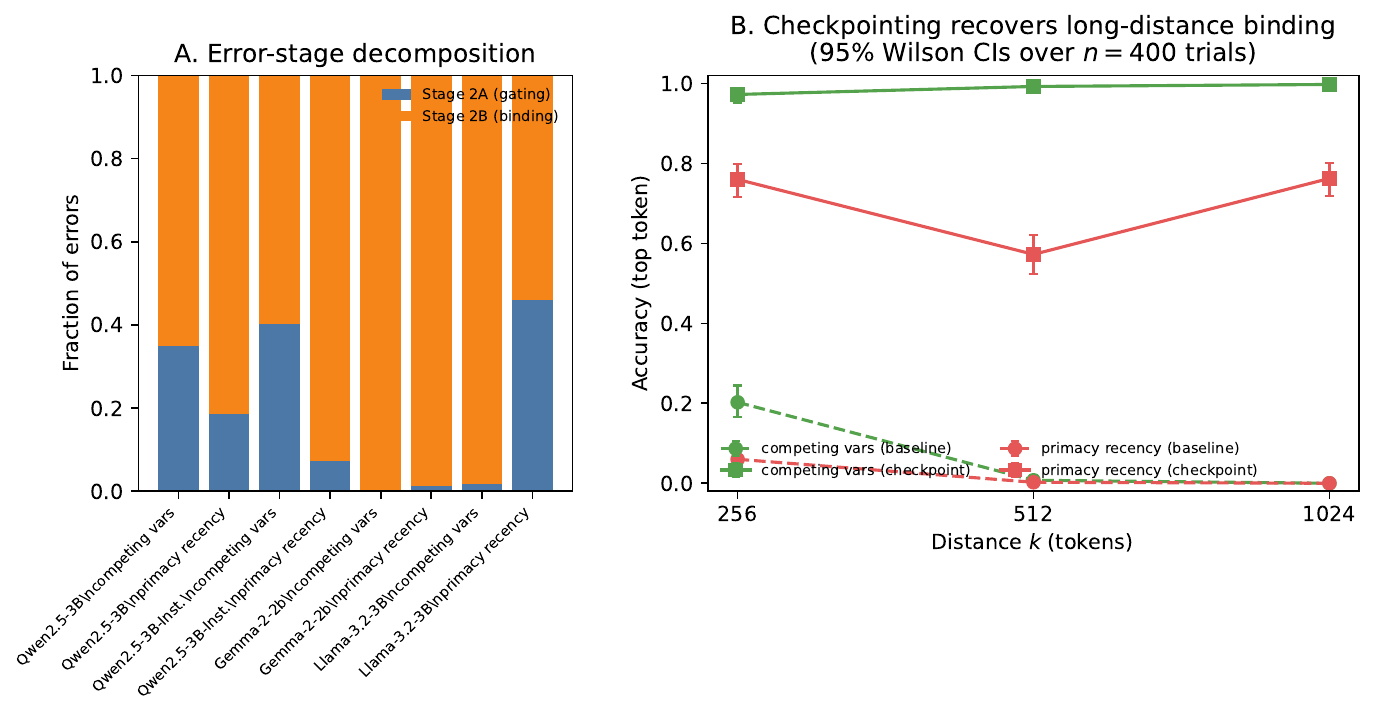}
\caption{\textbf{Spotlight summary of core empirical claims.} \textbf{(A)} Stage decomposition for representative hard-regime settings: most errors are Stage~2B (binding) rather than Stage~2A (gating), i.e., the model enters answer mode but selects the wrong candidate. \textbf{(B)} Checkpointing (restating the true binding every 128 tokens near the query) substantially recovers long-distance binding for Qwen2.5-3B; on \textsc{competing\_vars} at $k=1024$, it converts 0/400$\to$399/400 correct. Error bars are 95\% Wilson binomial confidence intervals over $n=400$ trials per cell.}
\label{fig:spotlight_results}
\end{figure*}

\paragraph{Models and tasks.}
For stage decomposition and probing, we evaluate Qwen2.5-3B (and Qwen2.5-3B-Instruct), Llama-3.2-3B, and Gemma-2-2b.
For checkpointing and scaling controls, we additionally evaluate Llama-3.2-3B-Instruct and Gemma-2-2b-it.
We study \textsc{competing\_vars} and \textsc{primacy\_recency} tasks with distances $k \in \{256, 512, 1024, 2048\}$ tokens.
Sample sizes are $n=800$ for stage decomposition and probing, $n=400$ for baseline comparisons and checkpointing, and $n\approx 160$ for patching (paired clean/corrupt runs).

\subsection{Prediction 1: Stage~2B Dominates}

Table~\ref{tab:stage2ab_multimodel} shows the stage decomposition across models and tasks. In \textsc{competing\_vars}, Stage~2B accounts for 65--100\% of errors depending on model and distance. The model enters answer mode but selects the wrong candidate.

\begin{table*}[t]
\centering
{\small
\setlength{\tabcolsep}{3pt}
\renewcommand{\arraystretch}{1.15}
\begin{tabular*}{\textwidth}{@{\extracolsep{\fill}}llrrrrrrr@{}}
\toprule
Model & Task & $k$ & Acc & \shortstack{Cand\\Acc} & \shortstack{Frac\\2A} & \shortstack{Frac\\2B} & GateGap & ValueGap\\
\midrule
Qwen2.5-3B & \textsc{competing} & 256 & 0.107 & 0.158 & 0.350 & 0.650 & 0.56 & -1.81 \\
Qwen2.5-3B & \textsc{primacy} & 256 & 0.043 & 0.045 & 0.185 & 0.815 & 0.96 & -3.40 \\
Qwen2.5-3B-Inst. & \textsc{competing} & 256 & 0.276 & 0.360 & 0.402 & 0.598 & 0.61 & -0.77 \\
Qwen2.5-3B-Inst. & \textsc{primacy} & 256 & 0.061 & 0.059 & 0.072 & 0.928 & 2.10 & -3.96 \\
Gemma-2-2B & \textsc{competing} & 256 & 0.995 & 0.995 & 0.000 & 1.000 & 3.03 & 2.28 \\
Gemma-2-2B & \textsc{primacy} & 256 & 0.801 & 0.804 & 0.013 & 0.987 & 1.65 & 0.47 \\
Llama-3.2-3B & \textsc{competing} & 2048 & 0.385 & 0.369 & 0.018 & 0.982 & 2.01 & -0.28 \\
Llama-3.2-3B & \textsc{primacy} & 2048 & 0.001 & 0.014 & 0.461 & 0.539 & -0.08 & -5.14 \\
\bottomrule
\end{tabular*}
}
\caption{Stage decomposition ($n=800$, greedy next-token decoding; bracket prompt style; \texttt{decoy\_heavy} filler; $|\Cset|\approx 50$ single-token candidates). Acc is $\Pp(\hat Y=V)$ (top token), Cand-Acc is $\Pp(Y=V)$ (best candidate). \textbf{Frac-2A/Frac-2B are fractions of errors} (conditional on $\hat Y\neq V$), decomposing errors into gating (non-candidate output) vs.\ binding (wrong candidate) failures. For Llama we report $k=2048$ to reach comparable difficulty.}
\label{tab:stage2ab_multimodel}
\end{table*}

Baseline accuracy shows a robustness hierarchy: Gemma $>$ Llama $\gg$ Qwen, with Qwen collapsing by $k=256$ on both tasks and all models eventually failing on \textsc{primacy\_recency} at sufficient distance (Appendix Table~\ref{tab:baseline_comparison}).

\subsection{Prediction 2: Probes Certify Information Presence}

On trials where the model outputs the wrong candidate (Stage~2B errors), we train a linear probe on the final-layer residual stream to predict the correct answer. Table~\ref{tab:probe_stage2b} shows the results.

\begin{table*}[t]
\centering
{\small
\setlength{\tabcolsep}{4pt}
\renewcommand{\arraystretch}{1.15}
\begin{tabular*}{\textwidth}{@{\extracolsep{\fill}}lrrrrrrr@{}}
\toprule
Task & $k$ & Acc & \shortstack{Cand\\Acc} & \shortstack{Frac\\2B} & Probe@0 & Probe@18 & Probe@35\\
\midrule
\textsc{competing} & 256 & 0.200 & 0.255 & 0.641 & 0.019 & 0.083 & 0.739 \\
\textsc{competing} & 512 & 0.000 & 0.000 & 0.968 & 0.008 & 0.029 & 0.390 \\
\textsc{primacy} & 256 & 0.069 & 0.074 & 0.754 & 0.020 & 0.040 & 0.354 \\
\textsc{primacy} & 512 & 0.000 & 0.000 & 0.994 & 0.020 & 0.020 & 0.124 \\
\bottomrule
\end{tabular*}
}
\caption{Probing on Stage~2B errors (Qwen2.5-3B, $n=800$; bracket prompt style; \texttt{decoy\_heavy} filler; $|\Cset|\approx 50$). \textbf{Frac-2B is a fraction of errors} (conditional on $\hat Y\neq V$) attributable to Stage~2B binding failures. Layers 0/18/35 are embedding, mid, and final residual streams. Chance is $\approx$2\%. At $k=256$, the final-layer probe recovers the correct answer on 74\% of error trials.}
\label{tab:probe_stage2b}
\end{table*}
	
	At $k=256$ in \textsc{competing\_vars}, the final-layer probe achieves 0.739 accuracy, which is 37$\times$ above chance, even though the model output the wrong answer. This shows that the correct value remains decodable from the hidden state on error trials, supporting the ``present but not used'' hypothesis.
	
	We can make the magnitude of this gap more concrete by converting accuracies into the Fano expression in \cref{thm:fano} (with $|\Cset|\approx 50$).
	In this regime, Cand-Acc $=0.255$ implies $I_{\mathrm{used}}\gtrsim 0.45$ nats. For intuition, applying the same expression to the final-layer probe's 0.739 accuracy on Stage~2B errors yields $\approx 2.32$ nats; since this is conditioned on errors (and thus may change the entropy of $V$), we treat it as an interpretable proxy rather than a formal bound (Appendix Table~\ref{tab:routing_mi_bounds}).

\subsection{Mechanistic Localization: Activation Patching}

We localize the routing failure using activation patching. For each layer and component (attention vs.\ MLP), we patch activations from a ``clean'' run (where the model would be correct) into a ``corrupt'' run (where it errs) and measure how much the correct answer's logit margin recovers.

Across three model families, we observe a consistent motif: late attention layers restore correct bindings, while late MLPs corrupt them (complete layer-component results in Appendix Table~\ref{tab:full_layers}). We also identify model-specific heads; notably, Gemma exhibits an ``anti-recency'' head that counterbalances a misbinding head (Appendix Table~\ref{tab:gemma_anti_recency_head}).

\subsection{Prediction 3: Checkpointing Recovers Accuracy}

To test whether failures reflect distance rather than incapacity, we introduce oracle checkpointing: restating bindings every 128 tokens. Table~\ref{tab:checkpoint_results} shows dramatic recovery.
To guard against an alternative explanation (format repetition rather than content-specific routing), our checkpointing suite also defines sham and wrong-checkpoint controls (Appendix~\ref{sec:checkpoint_controls}); in this paper we report oracle checkpointing results.

\begin{table*}[t]
\centering
{\small
\setlength{\tabcolsep}{6pt}
\renewcommand{\arraystretch}{1.15}
\begin{tabular*}{\textwidth}{@{\extracolsep{\fill}}llrrr@{}}
\toprule
Model & Task & $k$ & Baseline & +Checkpoint \\
\midrule
Qwen2.5-3B & \textsc{competing} & 1024 & 0.000 & \textbf{0.998} \\
Qwen2.5-3B & \textsc{primacy} & 1024 & 0.000 & \textbf{0.763} \\
Llama-3.2-3B-Inst. & \textsc{competing} & 2048 & 0.000 & \textbf{0.853} \\
Llama-3.2-3B-Inst. & \textsc{primacy} & 1024 & 0.003 & \textbf{0.915} \\
Gemma-2-2B-it & \textsc{primacy} & 1024 & 0.440 & \textbf{0.968} \\
\bottomrule
\end{tabular*}
}
\caption{Checkpointing summary ($n=400$; greedy next-token decoding; bracket prompt style; \texttt{decoy\_heavy} filler; $|\Cset|\approx 50$). Oracle checkpointing inserts a compact serialization of the true binding every 128 tokens in the tail filler, and can convert near-zero accuracy into near-perfect accuracy at long distance; see Appendix Table~\ref{tab:checkpoint_results_full} for full results. Tasks are \textsc{competing\_vars} (\textsc{competing}) and \textsc{primacy\_recency} (\textsc{primacy}).}
\label{tab:checkpoint_results}
\end{table*}

Qwen recovers from 0\% to 99.8\% accuracy at $k=1024$. The one exception is Llama at $k=2048$ on \textsc{primacy\_recency}, where checkpointing fails (0\% $\to$ 0.3\%). Inspection reveals that Stage~2A gating has collapsed: the model no longer enters answer mode. This confirms the stagewise picture: Stage~2B failures are distance-limited and recoverable; Stage~2A failures at extreme distance are a distinct breakdown.

\paragraph{Naturalistic binding at 8B--9B scale.}
We also evaluate an option-randomized \texttt{notes\_binding} task at 8B--9B scale.
Checkpointing partially recovers Meta-Llama-3.1-8B-Instruct at long distance and yields near-perfect performance for Gemma-2-9b-it (Appendix Table~\ref{tab:notes_binding_scale}).

\paragraph{Extension: audited reasoning traces.}
We extend pseudo-priors to chain-of-thought auditing by scrubbing cited evidence and comparing required vs.\ observed information budgets; details and full results are in Appendix~\cref{sec:cot}.

\paragraph{Toolkit and reproducibility.}
An anonymized reproducibility package (scripts + cached outputs) is included in the supplemental material and is sufficient to regenerate the reported tables and Figure~\ref{fig:spotlight_results}; see Appendix~\cref{sec:toolkit} for commands and environment notes.

\section{Discussion}
\label{sec:discussion}

\paragraph{Procedural versus factual hallucination.}
Our results suggest that many structured-generation failures are not about missing knowledge but about \emph{mis-commitment}. The model encodes the correct answer (probes recover it) but routes its output toward a biased competitor. This is qualitatively different from factual hallucination and requires different interventions.

\paragraph{Why the theory gives more than necessary conditions.}
Our theory uses standard information-theoretic tools, but aims to make them actionable for diagnosing LLM failures. Fano's inequality is a necessary condition: low information implies high error. Propositions~\ref{prop:mano_tight} and~\ref{prop:fano_slack} characterize when this necessary condition is close to sufficient (error rates near the minimax channel) and provide a measurable slack decomposition. The Bernoulli decompression bound (\cref{thm:bernoulli_projection}) is a tight specialization that yields closed-form ``bits-to-trust'' costs.

\paragraph{Mitigation strategies.}
Our results point to two levers. First, \emph{increase availability}: checkpointing, retrieval, and shorter contexts all increase $I_{\mathrm{avail}}$. Second, \emph{increase routing efficiency}: the patching results suggest that late MLPs are a bottleneck. Gemma's anti-recency head shows that architectural solutions exist; whether they can be trained or induced is an open question.

\section{Limitations}
\label{sec:limitations}

Our mechanistic claims (probing, patching) require activation access and thus apply directly only to open-weight models. For hosted APIs, we provide output-level and trace-level diagnostics but cannot estimate $I_{\mathrm{avail}}$ (and hence $\eta$) without activations; in this paper we therefore emphasize bounds and certificates derived from error rates and probes.

Trace auditing detects certificate failures, that is, claims that lack sufficient evidential support, but does not prove faithfulness to the model's hidden chain-of-thought. A model could produce a valid-looking trace via post-hoc rationalization.

Our experiments focus on synthetic binding tasks designed to isolate the phenomena. Whether the same mechanisms explain failures in naturalistic long-context tasks (e.g., document QA) remains to be validated.

\section{Broader Impact and Ethics}
\label{sec:broader_impact}

\paragraph{Potential positive impacts.}
Procedural hallucinations can undermine trust in systems that rely on long-context reasoning and tool-augmented generation. By decomposing failures into interpretable readout stages and providing concrete diagnostics, this work may help practitioners (i) detect when a model ``knows but does not use'' needed information, (ii) design mitigations such as checkpointing or answer-gating, and (iii) audit reasoning traces for unsupported claims.

\paragraph{Potential negative impacts and misuse.}
Diagnostic tools for model behavior can potentially be repurposed to find brittle points or to craft adversarial prompts that exploit model biases. We mitigate this by releasing no new model weights and by focusing on synthetic tasks and analysis code; users should apply the toolkit responsibly and in accordance with the underlying model and API terms of use.

\paragraph{Privacy and human subjects.}
Our experiments use synthetic prompts and do not involve human subjects or personal data. When using the optional API-based trace-auditing components, practitioners should avoid sending sensitive data to third-party services unless permitted and appropriate.

\section{Conclusion}

We have presented a rigorous framework for understanding procedural hallucinations: failures where the model possesses information but does not use it. The framework decomposes errors into gating (Stage~2A) and binding (Stage~2B) failures, formalizes ``present but not used'' via information-theoretic routing efficiency, provides tight bounds connecting error rates to information budgets, and extends to auditing reasoning traces.

Empirically, Stage~2B errors dominate in binding tasks. Probes certify that correct information is encoded on error trials. Activation patching localizes failures to late MLPs. Checkpointing can substantially recover long-distance binding by shortening the evidence path.

An anonymized toolkit and reproducibility package are included in the supplemental material, and we plan to release them publicly upon acceptance.

\clearpage
\bibliography{references}
\bibliographystyle{icml2026}

\onecolumn
\appendix

\section{Strawberry Counting Analysis}
\label{sec:strawberry_appendix}

Prompt: ``How many r's are in `strawberry'?'' The correct word-level statistic is $W=3$. For long traces, the final readout misbinds to competing statistics: WORD (intended), TOTAL (total count in the full trace), or SUFFIX (a recent suffix window). Empirically, baseline misbinds for $k>10$: at $k=20$ the final output is 6 (TOTAL), and at $k=30$ it is 8 (SUFFIX); larger $k$ yields TOTAL. A binding intervention that forces the final to match the word-level statistic returns final $=3$ for all tested $k$ with perfect trace fidelity (character match and run-count consistency both 1.00).
\section{Proofs}
\label{sec:proofs}

\subsection{Proof of \cref{prop:dpi_eta}}
\begin{proof}
Since $Y_k$ is a function of $H_k$, we have the Markov chain $V\to H_k \to Y_k$. By data processing, $\MI(V;Y_k)\le \MI(V;H_k)$. Nonnegativity of mutual information yields $0\le \eta_k\le 1$.
\end{proof}

\subsection{Proof of \cref{thm:fano} and \cref{prop:mano_tight}}
\begin{proof}[Proof sketch]
For uniform $V$ on $\Cset$, $H(V)=\log M$. Fano's inequality gives $H(V\mid Y)\le h(\varepsilon)+\varepsilon\log(M-1)$. Therefore $\MI(V;Y)=H(V)-H(V\mid Y)\ge \log M - h(\varepsilon)-\varepsilon\log(M-1)$.

Tightness: the $M$-ary symmetric channel has $P(Y=V)=1-\varepsilon$ and $P(Y=y\ne V)=\varepsilon/(M-1)$. For this channel, $H(V\mid Y)$ attains the Fano upper bound. Minimax optimality follows because the symmetric channel maximizes $H(V\mid Y)$ given $\varepsilon$.
\end{proof}

\subsection{Proof of \cref{prop:fano_slack}}
\begin{proof}
Let $E=\Ind\{Y\neq V\}$. Since $E$ is deterministic given $(V,Y)$, we have $\Hh(E\mid V,Y)=0$. By chain rule:
\[
\Hh(V\mid Y)=\Hh(E,V\mid Y)-\Hh(E\mid V,Y)=\Hh(E\mid Y)+\Hh(V\mid Y,E).
\]
Since $\Hh(V\mid Y,E=0)=0$ (when correct, $V=Y$), we have $\Hh(V\mid Y,E)=\varepsilon\,\Hh(V\mid Y,E=1)$. Thus:
\[
\Hh(V\mid Y)=\Hh(E\mid Y)+\varepsilon\,\Hh(V\mid Y,E=1).
\]
Using $\MI(V;Y)=\log M-\Hh(V\mid Y)$ and adding/subtracting Fano bounds yields \eqref{eq:fano_slack_decomp}. Nonnegativity: $\Hh(E\mid Y)\le h(\varepsilon)$ by Jensen; $\Hh(V\mid Y,E=1)\le \log(M-1)$ by support size.
\end{proof}

\subsection{Proof of \cref{thm:sdpi}}
\begin{proof}
Express mutual information as average KL:
\[
\MI(U;X)=\E_{u}\big[\KL(P(X\mid U=u)\,\|\,P(X))\big].
\]
Since $X\mapsto X'$ is channel $K$: $P(X'\mid U=u)=P(X\mid U=u)K$ and $P(X')=P(X)K$. By definition of $\alpha(K)$:
\[
\KL(P(X'\mid U=u)\,\|\,P(X'))\le \alpha(K)\,\KL(P(X\mid U=u)\,\|\,P(X)).
\]
Averaging yields $\MI(U;X')\le \alpha(K)\MI(U;X)$. Chain bound follows by iteration.
\end{proof}

\subsection{Proof of \cref{thm:bernoulli_projection}}
\begin{proof}[Proof sketch]
Let $f(\omega)=\Ind\{\omega\in A\}$. By data processing for KL:
\[
\KL(P\|Q)\ge \KL(P\circ f^{-1}\,\|\,Q\circ f^{-1}) = \KL(\ber(p)\|\ber(\tilde p)).
\]
Tightness: choose $P$ as the I-projection of $Q$ onto $\{P: P(A)=p\}$.
\end{proof}

\subsection{Proof of \cref{prop:unused_info}}
\begin{proof}
By data processing, $\MI(V;f(H_k))\le \MI(V;H_k)$. Subtracting $\MI(V;Y_k)$ yields the result.
\end{proof}

\begin{proposition}[Exact contraction for copy-or-noise]
\label{prop:copy_noise}
For the copy-or-noise channel with parameter $\alpha$ (copies input with probability $\alpha$, else outputs noise from $\nu$):
\[
\MI(U;X') = \alpha\,\MI(U;X).
\]
\end{proposition}
\begin{proof}
Let $B\sim\ber(\alpha)$ be independent of $(U,X)$ and $Z\sim \nu$ be independent noise. Since $B$ is independent of $U$:
\[
\MI(U;X')=\MI(U;X'\mid B)=\alpha\,\MI(U;X'\mid B=1)+(1-\alpha)\,\MI(U;X'\mid B=0).
\]
When $B=0$, $X'=Z$ independent of $U$, so $\MI(U;X'\mid B=0)=0$. When $B=1$, $X'=X$, so $\MI(U;X'\mid B=1)=\MI(U;X)$.
\end{proof}

\section{Audited Reasoning Traces}
\label{sec:cot}

We extend the framework to reasoning traces. The model produces a chain-of-thought, and we audit whether each step is supported by the evidence it cites.

\subsection{Trace Format}

We prompt the model to produce a structured trace $T = \{(c_i, S_i, o_i)\}_{i=1}^T$ where $c_i$ is an atomic claim, $S_i$ is a set of cited span identifiers, and $o_i$ is an optional confidence. A separate verifier labels each step as \texttt{ENTAILED}, \texttt{CONTRADICTED}, \texttt{NOT\_IN\_CONTEXT}, or \texttt{UNVERIFIABLE}. An executor then derives the answer using only verified steps.

\subsection{Trace-Level Pseudo-Priors}

For each step $i$, we define a pseudo-prior by scrubbing its cited spans:
\[
\tilde W^{(i)} := \doop(\text{spans in }S_i := \texttt{[REDACTED]}).
\]
Let $p_{0,i}$ be the verifier's probability that step $i$ is entailed under $\tilde W^{(i)}$, and $p_{1,i}$ be the probability under the full context. We define:
\[
\mathrm{ReqBits}_i := \KL(\ber(o_i)\,\|\,\ber(p_{0,i})),\qquad
\mathrm{ObsBits}_i := \KL(\ber(p_{1,i})\,\|\,\ber(p_{0,i})).
\]
If $\mathrm{ObsBits}_i < \mathrm{ReqBits}_i$, the step is \textbf{under-budget}: the cited evidence does not justify the claimed confidence.

\subsection{Empirical Validation}

On synthetic binding tasks (Mistral-7B-Instruct, $k\in\{1024,2048\}$, $n=300$), the audit isolates errors effectively. At $k=2048$, the pass subset achieves 99.5\% accuracy while the flagged subset drops to 63.7\% (Table~\ref{tab:trace_synth_budget}).

\begin{table}[t]
\centering
\small
\begin{tabular}{lrrrrrr}
\toprule
Setting & $k$ & Acc(\%) & Pass(\%) & Acc(pass)(\%) & Acc(flag)(\%) & Lift(pp) \\
\midrule
Mistral-7B (baseline) & 1024 & 98.3 & 54.8 & 100.0 & 97.0 & 3.0 \\
Mistral-7B (baseline) & 2048 & 87.0 & 68.6 & 99.5 & 63.7 & 35.8 \\
\bottomrule
\end{tabular}
\caption{Trace-budget audit ($\tau=0.75$). At large distance, errors concentrate in flagged traces.}
\label{tab:trace_synth_budget}
\end{table}

On QuALITY reading comprehension, we use an evidence-only verifier (since correct answers often require inference rather than literal entailment). Table~\ref{tab:quality_llm_audit_main} shows that passing the audit is strongly predictive of correctness: 81--83\% accuracy for pass vs.\ 59--62\% for flagged.

\begin{table}[t]
\centering
\small
\begin{tabular}{lrrrrrrr}
\toprule
Variant & $n$ & Pass(\%) & Acc(pass)(\%) & Acc(flag)(\%) & Lift(pp) & $\mathbb{E}[p_1]$ & $\mathbb{E}[p_0^{\mathrm{NE}}]$ \\
\midrule
baseline & 369 & 21.1 & 83.3 & 61.5 & 21.8 & 0.398 & 0.036 \\
rag & 378 & 25.1 & 81.1 & 60.4 & 20.6 & 0.428 & 0.039 \\
checkpoint & 383 & 24.0 & 82.6 & 59.5 & 23.2 & 0.430 & 0.037 \\
\bottomrule
\end{tabular}
\caption{QuALITY audit with null-family envelope certification ($\tau=0.75$, Llama-3.1-8B-Instruct). Pass means the budget test holds for all null operators $\nu\in\mathcal{N}$ (redact, delete, same-length mask, and no-evidence). $\mathbb{E}[p_0^{\mathrm{NE}}]$ reports the no-evidence pseudo-prior.}
\label{tab:quality_llm_audit_main}
\end{table}

Additional QuALITY robustness sweeps across $\tau$ and null operators appear in \cref{sec:quality_sweep}.

\section{Toolkit and Reproducibility}
\label{sec:toolkit}

\paragraph{Reproducibility package (anonymized).}
This submission includes an anonymized reproducibility package containing: (i) experiment scripts under \texttt{experiments/} (see also \texttt{experiments/results\_manifest.md}), (ii) the cached CSV/JSON outputs used to populate the reported tables, and (iii) a small plotting script (\texttt{experiments/spotlight\_make\_figures.py}, requires \texttt{matplotlib}) that regenerates Figure~\ref{fig:spotlight_results} from cached outputs. Core Python dependencies are listed in \texttt{experiments/requirements.txt}. We plan to release a public repository upon acceptance.

\paragraph{Toolkit.}
We include a toolkit that implements the diagnostics described in this paper for any API model exposing token logprobs.

The toolkit provides four capabilities. First, \textbf{Stage~2A/2B scoring}: given a prompt and candidate set, compute gate margin, value margin, and stage classification from the model's next-token logprobs. Second, \textbf{structured trace generation}: prompt the model to produce a JSON-formatted trace with span citations, using schema-constrained decoding. Third, \textbf{trace verification}: run a separate verifier model on each claim-span pair and compute $p_0$ (scrubbed) and $p_1$ (full context) from logprobs. Fourth, \textbf{budget computation}: calculate $\mathrm{ReqBits}$ and $\mathrm{ObsBits}$ for each trace step and flag under-budget claims.

We additionally include a \textbf{robustness suite} implementing option-randomized long-context binding (including $k=0$ controls and checkpointing) and output-only routing certificates.

\paragraph{LLM usage in the toolkit.}
The trace-auditing extension uses a separate verifier model queried via API logprobs (our implementation supports OpenAI's API, including \texttt{gpt-4o-mini}, and can be extended to other providers). The core binding experiments and mechanistic analyses in this paper use open-weight models locally; they do not require API models.

\paragraph{Existing assets and licenses/terms.}
We evaluate released model weights under their respective terms of use: Qwen2.5 models under the Qwen Research license (\url{https://huggingface.co/Qwen/Qwen2.5-3B/blob/main/LICENSE}); Llama 3.2 models under the Llama 3.2 Community License (\url{https://huggingface.co/meta-llama/Llama-3.2-3B/blob/main/LICENSE.txt}) and associated use policy (\url{https://huggingface.co/meta-llama/Llama-3.2-3B/blob/main/USE_POLICY.md}); and Gemma 2 models under Google's Gemma terms (\url{https://ai.google.dev/gemma/terms}).

\paragraph{Example reproduction commands.}
To regenerate Figure~\ref{fig:spotlight_results} from cached outputs:
\begin{lstlisting}
python experiments/spotlight_make_figures.py
\end{lstlisting}

To rerun a representative stage-decomposition slice (Prediction~1; Qwen2.5-3B, $k=256$, $n=800$ per task):
\begin{lstlisting}
torchrun --nproc_per_node 1 experiments/prediction1_stage2b/stage2b_multimodel_suite.py \
  --models Qwen/Qwen2.5-3B \
  --tasks competing_vars primacy_recency \
  --k_values 256 \
  --filler_types decoy_heavy \
  --trials_per_condition 800 \
  --batch_size 1 \
  --decoy_reps 12 \
  --n_distractors 48 \
  --dtype bf16 \
  --attn_impl sdpa \
  --outdir repro_out/stage2b \
  --out_csv stage2b_summary.csv \
  --out_json stage2b_summary.json
\end{lstlisting}

To rerun the checkpointing sweep used in Figure~\ref{fig:spotlight_results}B (Prediction~3; Qwen2.5-3B, $n=400$ per cell):
\begin{lstlisting}
torchrun --nproc_per_node 1 experiments/prediction3_checkpointing/checkpoint_mitigation_suite.py \
  --models Qwen/Qwen2.5-3B \
  --tasks competing_vars primacy_recency \
  --k_values 256 512 1024 \
  --filler_types decoy_heavy \
  --trials_per_k 400 \
  --batch_size 1 \
  --checkpoint_every 128 \
  --checkpoint_mode oracle \
  --dtype bf16 \
  --attn_impl sdpa \
  --out_csv repro_out/checkpointing.csv \
  --out_json repro_out/checkpointing.json
\end{lstlisting}

To rerun the mechanistic localization sweep (activation patching; Qwen2.5-3B, $k=256$):
\begin{lstlisting}
python experiments/mechanistic_localization/mech_stage2b_circuit_suite.py \
  --model Qwen/Qwen2.5-3B \
  --task competing_vars \
  --k 256 \
  --prompt_style bracket \
  --clean_filler random \
  --corrupt_filler decoy_heavy \
  --decoy_reps 12 \
  --n_pairs 40 \
  --layers last8 \
  --top_layers 2 \
  --top_heads 12 \
  --dtype bf16 \
  --attn_impl sdpa \
  --outdir repro_out/mech
\end{lstlisting}

\section{Experimental Protocols and Additional Results}
\label{sec:protocols}

\subsection{Experimental Protocols}
\label{sec:exp_protocols}

\paragraph{Candidate sets and tokenization.}
All binding tasks in this paper are \emph{single-token} classification problems: the correct value $V$ is a single vocabulary token.
For each model, we build a model-specific candidate pool by filtering a fixed list of short English words (fruits/Greek letters/colors/etc.) to those that tokenize to exactly one non-special token under that model's tokenizer (preferring leading-space forms when applicable), yielding $|\Cset|\approx 50$ candidates.
We use a bracket prompt style (\texttt{KEY=[}\,$\cdot$\,\texttt{]}) to make the answer boundary explicit. In scripts where a string-based prompt is used, we include an extra delimiter (whitespace or a bracket separator) to avoid tokenizer-dependent merges at the value boundary; in token-ID–constructed prompts, values are inserted as standalone token IDs.
Unless otherwise stated, each example's candidate set contains the target, its main competitor(s), and $48$ additional distractors sampled from the candidate pool, so $|\Cset|\approx 50$ and chance accuracy is $\approx 2\%$.

\paragraph{Filler regimes.}
We vary the distance $k$ by inserting $k$ filler tokens between evidence and query.
We use four filler regimes: \texttt{repeat} (repeat a single filler token), \texttt{coherent} (repeat a short natural-language paragraph), \texttt{random} (sample random non-special tokens), and \texttt{decoy\_heavy} (mostly filler, but insert the competitor token \texttt{decoy\_reps} times at uniformly spaced positions, plus light random sprinkling).
Unless otherwise stated, main-text stage decomposition and probing results use \texttt{decoy\_heavy} with \texttt{decoy\_reps}=12 to induce Stage~2B errors.

\paragraph{Seeds and candidate pools.}
Candidate pools and distractor choices are seeded and thus can differ slightly across experiment suites; all reported results use fixed seeds and maintain $|\Cset|\approx 50$.

\paragraph{Tasks and variants.}
In addition to \textsc{competing\_vars} and \textsc{primacy\_recency} (defined in \cref{sec:stage}), we use \textsc{decoy\_injection}, where a single key is assigned once to a target value but the filler contains many repetitions of a \emph{decoy} competitor token; this isolates competition from frequency alone.
For completeness, the code also includes strawberry-style binding tasks (Appendix~\ref{sec:strawberry_appendix}) and a factorized prompt variant that reduces misbinding by explicitly separating competing statistics.

\paragraph{Decoding and stage metrics.}
We evaluate greedy next-token decoding at the readout position.
We compute Acc as $\Pp(\hat Y=V)$ (top-vocabulary token), Cand-Acc as $\Pp(Y=V)$ (best candidate token), GateGap as $(\max_{v\in\Cset} z(v))-(\max_{x\notin\Cset} z(x))$, and ValueGap as $z(V)-\max_{v\in\Cset\setminus\{V\}} z(v)$.
Stage~2A vs.\ Stage~2B are defined by whether the model's top token $\hat Y$ is outside vs.\ inside $\Cset$ on error trials.

\paragraph{Uncertainty (confidence intervals).}
For headline accuracies, we report 95\% Wilson binomial confidence intervals over the $n$ independent prompt instances in each cell (e.g., Figure~\ref{fig:spotlight_results} uses $n=400$). These CIs reflect \emph{instance-to-instance} variability under the stated prompt generator and fixed seeds; they do not capture variability from retraining or changing model checkpoints.

\paragraph{Compute and resources.}
All reported experiments are inference-only on released models (no training or fine-tuning of foundation models). As a hardware-agnostic proxy, we report the number of evaluated prompts and the maximum distance $k$.
Reference configuration for reproduction is a single NVIDIA A100 40GB (e.g., Colab A100) with \texttt{batch\_size=1} as in the Appendix commands; wall-clock time depends on hardware and scales approximately linearly with the number of evaluated prompts.
The Stage~2B suite cached in \texttt{experiments/prediction1\_stage2b/} evaluates 48{,}000 prompts total (60 conditions $\times$ $n=800$) with $k$ up to 2048.
The checkpointing suites cached in \texttt{experiments/prediction3\_checkpointing/} evaluate 9{,}600 prompts total across Qwen and Gemma (12 evaluated conditions $\times$ $n=400$ per model: 2 tasks $\times$ 3 $k$ values $\times$ \{baseline, checkpoint\}), with $k$ up to 1024.
Mechanistic patching uses paired clean/corrupt prompts (typically $n\approx 160$ pairs per model/setting; see Appendix tables).
The provided scripts run on NVIDIA GPUs via PyTorch/\texttt{torchrun}; the mechanistic review-pack scripts are designed to be runnable on a Colab A100-class GPU (see script headers).

\paragraph{Probing.}
For probe results, we train a linear softmax probe on hidden states at the readout position to predict the target token ID.
For each (task, $k$, filler) condition we collect $n$ trials, split them 50/50 into train/test, train for 100 full-batch gradient steps with $L_2$ regularization, and report test accuracy specifically on Stage~2B error trials (and on correct trials as a sanity check).

\paragraph{Activation patching.}
For mechanistic localization, we build paired \emph{clean/corrupt} prompts that share the same key/value tokens and positions and differ only in the filler distribution (clean: \texttt{random}; corrupt: \texttt{decoy\_heavy} with the competitor injected).
We patch clean activations into the corrupt run and measure restoration of the target-vs-competitor logit margin at the readout position.
At the layer-component level, we patch either (i) the concatenated attention-head output vector (the input to \texttt{o\_proj}) or (ii) the MLP output vector.
At the head level, we patch or ablate individual head slices at the \texttt{o\_proj} input.

\subsection{Head-Level Causal Analysis: Gemma ``Anti-Recency'' Head}
\label{sec:gemma_anti_recency}

Table~\ref{tab:gemma_anti_recency_head} provides a concrete example of a model-specific head-level motif in Gemma-2-2b-it on \textsc{primacy\_recency} at $k=256$.
We find a head whose ablation \emph{improves} the target-vs-competitor margin (a misbinding/interference head) and a head whose ablation \emph{hurts} the margin (a counterbalancing head that resists recency), motivating the informal term ``anti-recency.''

\begin{table}[t]
\centering
\small
\setlength{\tabcolsep}{6pt}
\begin{tabular}{lrrr}
\toprule
Head & Mean patch restoration & Mean ablation $\Delta$margin & $n$ \\
\midrule
L25H2 (misbinding/interference) & 0.169 & +0.508 & 159 \\
L22H3 (counterbalances recency) & 0.121 & -0.321 & 159 \\
\bottomrule
\end{tabular}
\caption{Gemma-2-2b-it head-level patching/ablation on \textsc{primacy\_recency} ($k=256$; bracket prompt style; clean filler \texttt{random}; corrupt filler \texttt{decoy\_heavy} with \texttt{decoy\_reps}=12). Patch restoration is normalized margin recovery; ablation $\Delta$margin is $m_{\mathrm{ablated}}-m_{\mathrm{corrupt}}$ (positive means ablating the head helps correct-vs-competitor margin; negative means ablation hurts).}
\label{tab:gemma_anti_recency_head}
\end{table}

\subsection{Baseline Accuracy}

\begin{table}[t]
\centering
\small
\begin{tabular}{llrrr}
\toprule
Task & $k$ & Qwen2.5-3B & Llama-3.2-3B-Inst. & Gemma-2-2b-it \\
\midrule
\multirow{4}{*}{competing\_vars} 
& 256 & 0.203 & \textbf{0.993} & \textbf{1.000} \\
& 512 & 0.008 & \textbf{0.993} & \textbf{1.000} \\
& 1024 & 0.000 & 0.455 & \textbf{0.998} \\
& 2048 & --- & 0.000 & --- \\
\midrule
\multirow{4}{*}{primacy\_recency} 
& 256 & 0.060 & \textbf{0.868} & \textbf{0.963} \\
& 512 & 0.003 & 0.553 & \textbf{0.798} \\
& 1024 & 0.000 & 0.003 & 0.440 \\
& 2048 & --- & 0.000 & --- \\
\bottomrule
\end{tabular}
\caption{Baseline accuracy ($n=400$; greedy next-token decoding; bracket prompt style; \texttt{decoy\_heavy} filler; $|\Cset|\approx 50$ single-token candidates). All models eventually fail on \textsc{primacy\_recency} at sufficient distance.}
\label{tab:baseline_comparison}
\end{table}

\subsection{Checkpointing (Full Results)}

\begin{table*}[t]
\centering
\small
\setlength{\tabcolsep}{4pt}
\begin{tabular}{llrrrrrr}
\toprule
Model & Task & $k$ & Baseline & +Checkpoint & $\Delta$Acc & ValueGap$_\text{base}$ & ValueGap$_\text{chk}$ \\
\midrule
\multirow{6}{*}{Qwen2.5-3B} 
& competing\_vars & 256 & 0.203 & \textbf{0.973} & +0.770 & -1.02 & +2.81 \\
& competing\_vars & 512 & 0.008 & \textbf{0.993} & +0.985 & -5.75 & +2.61 \\
& competing\_vars & 1024 & 0.000 & \textbf{0.998} & +0.998 & -9.93 & +2.88 \\
\cmidrule{2-8}
& primacy\_recency & 256 & 0.060 & \textbf{0.760} & +0.700 & -2.88 & +0.80 \\
& primacy\_recency & 512 & 0.003 & \textbf{0.573} & +0.570 & -6.53 & +0.26 \\
& primacy\_recency & 1024 & 0.000 & \textbf{0.763} & +0.763 & -9.05 & +1.01 \\
\midrule
\multirow{7}{*}{Llama-3.2-3B-Inst.} 
& competing\_vars & 256 & 0.993 & 0.998 & +0.005 & +8.89 & +9.20 \\
& competing\_vars & 1024 & 0.455 & \textbf{0.998} & +0.543 & +0.03 & +8.96 \\
& competing\_vars & 2048 & 0.000 & \textbf{0.853} & +0.853 & -3.78 & +1.44 \\
\cmidrule{2-8}
& primacy\_recency & 256 & 0.868 & \textbf{0.995} & +0.127 & +1.81 & +3.48 \\
& primacy\_recency & 512 & 0.553 & \textbf{0.973} & +0.420 & +0.16 & +2.95 \\
& primacy\_recency & 1024 & 0.003 & \textbf{0.915} & +0.912 & -4.30 & +1.64 \\
& primacy\_recency & 2048 & 0.000 & 0.003 & +0.003 & -4.02 & -1.35 \\
\midrule
\multirow{6}{*}{Gemma-2-2b-it} 
& competing\_vars & 256 & 1.000 & 1.000 & 0.000 & +8.82 & +7.98 \\
& competing\_vars & 512 & 1.000 & 1.000 & 0.000 & +7.72 & +8.10 \\
& competing\_vars & 1024 & 0.998 & 1.000 & +0.003 & +6.55 & +7.71 \\
\cmidrule{2-8}
& primacy\_recency & 256 & 0.963 & \textbf{0.993} & +0.030 & +2.74 & +3.80 \\
& primacy\_recency & 512 & 0.798 & \textbf{0.963} & +0.165 & +1.37 & +2.36 \\
& primacy\_recency & 1024 & 0.440 & \textbf{0.968} & +0.528 & -0.62 & +2.70 \\
\bottomrule
\end{tabular}
\caption{Checkpointing results ($n=400$; greedy next-token decoding; bracket prompt style; \texttt{decoy\_heavy} filler; $|\Cset|\approx 50$). Checkpoints insert a compact serialization of the true binding every 128 tokens in the tail filler. Qwen recovers from 0\% to 99.8\% at $k=1024$. Exception: Llama at $k=2048$ on \textsc{primacy\_recency} shows checkpoint failure, as Stage~2A gating has collapsed and cannot be recovered by re-statement.}
\label{tab:checkpoint_results_full}
\end{table*}

\subsection{Checkpointing Controls}
\label{sec:checkpoint_controls}

Checkpointing could in principle help by repeating formatting cues rather than by injecting the correct content.
To separate these explanations, our checkpointing suite defines two simple controls:
\textbf{sham checkpointing}, which inserts an irrelevant but similarly formatted checkpoint statement, and
\textbf{wrong checkpointing}, which inserts a checkpoint that states the competitor/decoy value.
Under a content-specific routing interpretation, we would expect sham checkpointing to provide little benefit, while wrong checkpointing should systematically steer the output toward the wrong value.
The experiment script supports these conditions via \texttt{--checkpoint\_mode \{oracle,sham,wrong\}}; in this paper we report oracle checkpointing results.

\subsection{Naturalistic \texttt{notes\_binding} at 8B--9B Scale}

\begin{table}[t]
\centering
\small
\setlength{\tabcolsep}{4pt}
\begin{tabular}{r|ccc|ccc}
\toprule
& \multicolumn{3}{c|}{\textbf{Meta-Llama-3.1-8B-Instruct}} & \multicolumn{3}{c}{\textbf{Gemma-2-9b-it}} \\
$k$ & Acc (base) & $P[Z]$ (base) & Acc (+Chk) & Acc (base) & $P[Z]$ (base) & Acc (+Chk) \\
\midrule
0    & 0.435 & 0.185 & 0.510 & 0.330 & 0.490 & 1.000 \\
64   & 0.575 & 0.130 & 0.590 & 0.415 & 0.300 & 1.000 \\
256  & 0.550 & 0.155 & 0.550 & 0.640 & 0.150 & 0.995 \\
1024 & 0.435 & 0.105 & 0.580 & 0.740 & 0.015 & 1.000 \\
2048 & 0.430 & 0.055 & 0.600 & 0.820 & 0.000 & 1.000 \\
\bottomrule
\end{tabular}
\caption{Naturalistic \texttt{notes\_binding} at 8B--9B scale with randomized choice ordering ($n=200$ per cell).
$P[Z]$ is the abstention rate (Stage~2A failure).
Llama degrades at long distance and is partially recovered by checkpointing; Gemma is near-perfect under checkpointing and
exhibits a large reduction in abstention as $k$ increases.}
\label{tab:notes_binding_scale}
\end{table}

\begin{table}[h]
\centering
\scriptsize
\begin{tabular}{llrr|llrr|llrr}
\toprule
\multicolumn{4}{c|}{\textbf{Qwen2.5-3B-Instruct}} & \multicolumn{4}{c|}{\textbf{Llama-3.2-3B-Instruct}} & \multicolumn{4}{c}{\textbf{Gemma-2-2b-it}} \\
L & C & Rest & $\Delta$M & L & C & Rest & $\Delta$M & L & C & Rest & $\Delta$M \\
\midrule
28 & attn & -0.03 & -0.50 & 20 & attn & -0.04 & 0.04 & 18 & attn & 0.02 & -0.41 \\
28 & mlp & 0.03 & 0.59 & 20 & mlp & 0.04 & 0.03 & 18 & mlp & 0.01 & 0.17 \\
29 & attn & 0.06 & 1.02 & 21 & attn & 0.23 & 0.06 & 19 & attn & 0.10 & 0.12 \\
29 & mlp & -0.01 & -0.24 & 21 & mlp & -0.04 & 0.00 & 19 & mlp & -0.03 & -0.05 \\
30 & attn & -0.04 & -0.71 & 22 & attn & -0.03 & 0.09 & 20 & attn & -0.01 & 0.02 \\
30 & mlp & 0.07 & 1.34 & 22 & mlp & 0.01 & 0.01 & 20 & mlp & -0.01 & 0.01 \\
31 & attn & -0.01 & -0.19 & 23 & attn & -0.03 & -0.15 & 21 & attn & -0.09 & -0.55 \\
31 & mlp & 0.08 & 1.45 & 23 & mlp & 0.07 & 0.14 & 21 & mlp & 0.00 & 0.21 \\
32 & attn & 0.11 & 2.09 & 24 & attn & 0.02 & -0.04 & 22 & attn & 0.31 & 0.40 \\
32 & mlp & 0.05 & 0.81 & 24 & mlp & 0.05 & 0.05 & 22 & mlp & -0.01 & 0.02 \\
33 & attn & 0.18 & 3.44 & 25 & attn & 0.10 & 0.23 & 23 & attn & 0.03 & 0.03 \\
33 & mlp & -0.04 & -0.81 & 25 & mlp & -0.10 & -0.19 & 23 & mlp & -0.16 & 0.07 \\
34 & attn & 0.01 & 0.21 & 26 & attn & 0.21 & 0.24 & 24 & attn & -0.01 & 0.04 \\
34 & mlp & -0.11 & -1.99 & 26 & mlp & -0.01 & -0.22 & 24 & mlp & 0.02 & 0.10 \\
35 & attn & 0.01 & 0.10 & 27 & attn & 0.19 & 0.20 & 25 & attn & 0.36 & 0.98 \\
35 & mlp & -0.20 & -3.30 & 27 & mlp & 0.07 & 0.18 & 25 & mlp & -0.19 & -0.47 \\
\bottomrule
\end{tabular}
\caption{Complete layer-component patching results (paired clean/corrupt prompts; greedy next-token decoding; bracket prompt style; $k=256$; clean filler \texttt{random}; corrupt filler \texttt{decoy\_heavy} with \texttt{decoy\_reps}=12; $n\approx 160$ Stage~2B pairs per model). Rest is mean normalized restoration $(m_{\mathrm{patched}}-m_{\mathrm{corrupt}})/(m_{\mathrm{clean}}-m_{\mathrm{corrupt}})$ of the target-vs-competitor logit margin $m$; $\Delta M$ is the mean change in margin $(m_{\mathrm{patched}}-m_{\mathrm{corrupt}})$. Qwen uses \textsc{competing\_vars}; Llama/Gemma use \textsc{primacy\_recency}.}
\label{tab:full_layers}
\end{table}

\begin{table}[t]
\centering
\scriptsize
\begin{tabular}{llrrrrrrr}
\toprule
Task & Filler & $k$ & Acc & Cand-Acc & Frac-2A & Frac-2B & GateGap & ValueGap\\
\midrule
competing\_vars & decoy\_heavy & 128 & 0.999 & 0.999 & 0.000 & 1.000 & 3.74 & 5.12 \\
competing\_vars & decoy\_heavy & 256 & 0.107 & 0.158 & 0.350 & 0.650 & 0.56 & -1.81 \\
competing\_vars & decoy\_heavy & 512 & 0.003 & 0.001 & 0.023 & 0.977 & 1.99 & -7.26 \\
competing\_vars & decoy\_heavy & 1024 & 0.000 & 0.000 & 0.033 & 0.968 & 1.79 & -10.10 \\
competing\_vars & repeat & 128 & 0.030 & 0.080 & 0.570 & 0.430 & -0.07 & -3.39 \\
competing\_vars & repeat & 256 & 0.026 & 0.069 & 0.792 & 0.208 & -0.80 & -4.33 \\
competing\_vars & repeat & 512 & 0.000 & 0.094 & 0.988 & 0.013 & -2.97 & -3.46 \\
competing\_vars & repeat & 1024 & 0.000 & 0.249 & 1.000 & 0.000 & -7.76 & -1.83 \\
decoy\_injection & decoy\_heavy & 128 & 0.979 & 0.994 & 0.824 & 0.176 & 3.23 & 7.14 \\
decoy\_injection & decoy\_heavy & 256 & 0.240 & 0.907 & 0.993 & 0.007 & -1.17 & 2.41 \\
decoy\_injection & decoy\_heavy & 512 & 0.000 & 0.294 & 1.000 & 0.000 & -6.92 & -1.24 \\
decoy\_injection & decoy\_heavy & 1024 & 0.000 & 0.149 & 1.000 & 0.000 & -8.09 & -2.57 \\
decoy\_injection & repeat & 128 & 0.065 & 0.996 & 1.000 & 0.000 & -2.03 & 5.03 \\
decoy\_injection & repeat & 256 & 0.000 & 0.779 & 1.000 & 0.000 & -6.98 & 1.51 \\
decoy\_injection & repeat & 512 & 0.000 & 0.292 & 1.000 & 0.000 & -9.58 & -1.09 \\
decoy\_injection & repeat & 1024 & 0.000 & 0.295 & 1.000 & 0.000 & -10.53 & -1.33 \\
primacy\_recency & decoy\_heavy & 128 & 0.734 & 0.755 & 0.146 & 0.854 & 1.68 & 0.64 \\
primacy\_recency & decoy\_heavy & 256 & 0.043 & 0.045 & 0.185 & 0.815 & 0.96 & -3.40 \\
primacy\_recency & decoy\_heavy & 512 & 0.000 & 0.000 & 0.005 & 0.995 & 2.55 & -7.77 \\
primacy\_recency & decoy\_heavy & 1024 & 0.000 & 0.000 & 0.049 & 0.951 & 1.31 & -8.14 \\
primacy\_recency & repeat & 128 & 0.036 & 0.046 & 0.370 & 0.630 & 0.33 & -3.69 \\
primacy\_recency & repeat & 256 & 0.020 & 0.046 & 0.807 & 0.193 & -0.62 & -3.86 \\
primacy\_recency & repeat & 512 & 0.000 & 0.142 & 1.000 & 0.000 & -3.42 & -2.47 \\
primacy\_recency & repeat & 1024 & 0.000 & 0.139 & 1.000 & 0.000 & -7.06 & -2.48 \\
\bottomrule
\end{tabular}
\caption{Complete Stage~2A/2B results for Qwen2.5-3B ($n=800$ per row; greedy next-token decoding; bracket prompt style; $|\Cset|\approx 50$ single-token candidates). Acc is $\Pp(\hat Y=V)$ (top token), Cand-Acc is $\Pp(Y=V)$ (best candidate). \textbf{Frac-2A/Frac-2B are fractions of errors} (conditional on $\hat Y\neq V$).}
\label{tab:qwen_full}
\end{table}

\begin{table}[t]
\centering
\small
\setlength{\tabcolsep}{6pt}
\begin{tabular}{lrrrrr}
\toprule
Task & $k$ & Cand-Acc & $I_{\mathrm{used}}\ge$ & Probe@35 (Stage~2B err) & $I_{\mathrm{probe,err}}^{\mathrm{proxy}}$ \\
\midrule
competing\_vars & 256 & 0.255 & 0.45 & 0.739 & 2.32 \\
competing\_vars & 512 & 0.000 & 0.00 & 0.390 & 0.87 \\
primacy\_recency & 256 & 0.074 & 0.04 & 0.354 & 0.75 \\
primacy\_recency & 512 & 0.000 & 0.00 & 0.124 & 0.13 \\
\bottomrule
\end{tabular}
\caption{Converting accuracies into mutual-information quantities using the Fano expression $\Phi_M(\varepsilon)$ from \cref{thm:fano} with $M=50$ (candidate set size $\approx 50$; nats). For the model's candidate decision (Cand-Acc; unconditional, with $V$ uniform by construction), $\Phi_M$ is a valid lower bound on $I_{\mathrm{used}}$. For the final-layer probe accuracy restricted to Stage~2B error trials, we report the same quantity as an \emph{interpretable proxy} (conditioning can change $H(V)$, so the lower-bound interpretation need not hold). For finite-sample cells with Cand-Acc $<1/M$ we clamp the bound to 0.}
\label{tab:routing_mi_bounds}
\end{table}

\subsection{QuALITY Robustness}
\label{sec:quality_sweep}

Table~\ref{tab:quality_sweep} shows robustness of the \emph{null-family envelope} evidence-only audit across reliability targets $\tau\in\{0.65,0.70,0.75,0.80\}$.

\begin{table}[h]
\centering
\small
\begin{tabular}{r l r r r r}
\toprule
$\tau$ & Variant & Pass(\%) & Acc(pass)(\%) & Acc(flag)(\%) & Lift(pp) \\
\midrule
0.65 & baseline   & 24.7 & 83.5 & 60.4 & 23.1 \\
0.65 & checkpoint & 27.4 & 81.0 & 59.0 & 22.0 \\
0.65 & rag        & 28.6 & 79.6 & 60.0 & 19.6 \\
\midrule
0.70 & baseline   & 22.5 & 84.3 & 60.8 & 23.5 \\
0.70 & checkpoint & 25.3 & 82.5 & 59.1 & 23.4 \\
0.70 & rag        & 26.2 & 81.8 & 59.9 & 22.0 \\
\midrule
0.75 & baseline   & 21.1 & 83.3 & 61.5 & 21.8 \\
0.75 & checkpoint & 24.0 & 82.6 & 59.5 & 23.2 \\
0.75 & rag        & 25.1 & 81.1 & 60.4 & 20.6 \\
\midrule
0.80 & baseline   & 15.7 & 87.9 & 62.1 & 25.9 \\
0.80 & checkpoint & 20.1 & 83.1 & 60.5 & 22.7 \\
0.80 & rag        & 21.2 & 86.2 & 60.1 & 26.2 \\
\bottomrule
\end{tabular}
\caption{Robustness of the QuALITY null-family envelope audit across reliability targets $\tau$. Passing is consistently associated with substantially higher answer accuracy.}
\label{tab:quality_sweep}
\end{table}

\begin{table}[h]
\centering
\small
\begin{tabular}{l l r r r r}
\toprule
Null operator & Variant & Pass(\%) & Acc(pass)(\%) & Acc(flag)(\%) & Lift(pp) \\
\midrule
no-evidence & baseline & 22.8 & 79.8 & 62.1 & 17.7 \\
no-evidence & checkpoint & 25.1 & 81.2 & 59.6 & 21.7 \\
no-evidence & rag & 26.5 & 79.0 & 60.8 & 18.2 \\
redact-span & baseline & 27.4 & 81.2 & 60.4 & 20.7 \\
redact-span & checkpoint & 30.8 & 78.8 & 58.9 & 19.9 \\
redact-span & rag & 31.2 & 80.5 & 58.8 & 21.7 \\
delete-span & baseline & 29.3 & 82.4 & 59.4 & 23.0 \\
delete-span & checkpoint & 31.6 & 77.7 & 59.2 & 18.5 \\
delete-span & rag & 33.1 & 80.0 & 58.5 & 21.5 \\
mask-same-len & baseline & 26.0 & 84.4 & 59.7 & 24.7 \\
mask-same-len & checkpoint & 31.1 & 80.7 & 58.0 & 22.7 \\
mask-same-len & rag & 30.4 & 80.9 & 58.9 & 21.9 \\
envelope (all) & baseline & 21.1 & 83.3 & 61.5 & 21.8 \\
envelope (all) & checkpoint & 24.0 & 82.6 & 59.5 & 23.2 \\
envelope (all) & rag & 25.1 & 81.1 & 60.4 & 20.6 \\
\bottomrule
\end{tabular}
\caption{Sensitivity of QuALITY audit to the concrete pseudo-prior operator at $\tau=0.75$. ``Envelope'' requires the budget test to pass for all null operators; it provides a conservative, assumption-light certificate while retaining strong predictive lift.}
\label{tab:quality_null_family}
\end{table}

\begin{table}[h]
\centering
\small
\begin{tabular}{r r r r l}
\toprule
$k$ & $n$ & Pass(\%) & $\mathbb{E}[p_1]$ & $\mathbb{E}[p_0^{\min},p_0^{\max}]$ \\
\midrule
512 & 200 & 66.0 & 0.756 & [0.179, 0.272] \\
1024 & 200 & 0.0 & 0.728 & [0.097, 0.184] \\
2048 & 200 & 100.0 & 0.841 & [0.181, 0.361] \\
\bottomrule
\end{tabular}
\caption{Null-family statistics for a synthetic binding audit ($\tau=0.75$). Although $p_0$ varies across null operators, the pass/fail decision is invariant across the null family in this setting (envelope equals any single-null decision).}
\label{tab:binding_null_family}
\end{table}

\end{document}